\documentclass[10pt]{article}
\usepackage[utf8]{inputenc}
\usepackage{booktabs}
\usepackage{multirow}
\usepackage{graphicx}
\usepackage{amsmath,amsthm}
\usepackage{amsfonts}
\usepackage{amssymb}
\usepackage{bbm}
\usepackage{algorithm}
\usepackage[noend]{algpseudocode}

\usepackage{mathtools}
\mathtoolsset{showonlyrefs}

\newtheorem{manualtheoreminner}{Theorem}
\newenvironment{manualtheorem}[1]{%
  \renewcommand\themanualtheoreminner{#1}%
  \manualtheoreminner
}{\endmanualtheoreminner}

\newtheorem{theorem}{Theorem} 

\newtheorem{cor}{Corollary}[section]
\newtheorem{definition}{Definition}

\newcommand{\Xtest}{X_{\rm test}}
\newcommand{\Ytest}{Y_{\rm test}}
\newcommand{\Yhat}{\hat{Y}}

\newcommand{\R}{\mathbb{R}}
\renewcommand{\P}{\mathbb{P}}
\newcommand{\A}{\mathcal{A}}
\newcommand{\X}{\mathcal{X}}
\newcommand{\Y}{\mathcal{Y}}

\newcommand{\lhat}{\hat{\lambda}}
\newcommand{\T}{\mathcal{T}}
\newcommand{\Tlam}{\mathcal{T}_{\lambda}}

\newcommand{\ind}[1]{\mathbbm{1}\left\{#1\right\}}

\begin{document}
\title{ % placeholder title
    Improving Trustworthiness of AI Disease Severity Rating in Medical Imaging with Ordinal Conformal Prediction Sets
% \thanks{Supported by organization.}}
    % \vspace{-1.5cm}
}
\author{Charles Lu*, Anastasios N. Angelopoulos*, Stuart Pomerantz}
\date{}

\def\thefootnote{*}\footnotetext{Equal contribution}\def\thefootnote{\arabic{footnote}}

\maketitle              % typeset the header of the contribution
\begin{abstract}
    The regulatory approval and broad clinical deployment of medical AI have been hampered  by the perception that deep learning models fail in unpredictable and possibly catastrophic ways. A lack of statistically rigorous uncertainty quantification is a significant factor undermining trust in AI results. Recent developments in distribution-free uncertainty quantification present practical solutions for these issues by providing reliability guarantees for black-box models on arbitrary data distributions as formally valid finite-sample prediction intervals. Our work applies these new uncertainty quantification methods --- specifically conformal prediction --- to a deep-learning model for grading the severity of spinal stenosis in lumbar spine MRI. We demonstrate a technique for forming ordinal prediction sets that are guaranteed to contain the correct stenosis severity within a user-defined probability (confidence interval). On a dataset of 409 MRI exams processed by the deep-learning model, the conformal method provides tight coverage with small  prediction set sizes. Furthermore, we explore the potential clinical applicability of flagging cases with high uncertainty predictions (large prediction sets) by quantifying an increase in the prevalence of significant imaging abnormalities (e.g. motion artifacts, metallic artifacts, and tumors) that could degrade confidence in predictive performance when compared to a random sample of cases.
    % Our experiments represent a step towards the clinical application of conformal prediction to increase the trustworthiness of AI in medical imaging analysis.
\end{abstract}

\section{Introduction}
    \setlength{\textfloatsep}{0.2cm}
    \setlength{\floatsep}{0.2cm}
    \setlength{\intextsep}{0.2cm}
    Although many studies have demonstrated high overall accuracy in automating medical imaging diagnosis with deep-learning AI models, translation to actual clinical deployment has proved difficult. It is widely observed  that deep learning algorithms can fail in bizarre ways and with misplaced confidence~\cite{Minderer2021RevisitingTC,hendrycks2021nae}. A core problem is a lack of \emph{trust} --- a survey of radiologists found they that although they thought AI tools add value to their clinical practice, they would not trust AI for autonomous clinical use due to perceived and experienced unreliability~\cite{allen20212020}.
    
    Herein, we present methods for endowing arbitrary AI systems with formal mathematical guarantees that give clinicians explicit assurances about an algorithm's overall performance and most importantly for a given study. These guarantees are \emph{distribution-free}---they work for any (pre-trained) model, any (possibly unknown) data distribution, and in finite samples. Although such guarantees do not solve the issue of trust entirely, the precise and formal understanding of their model's predictive uncertainty enables a clinician to potentially work more assuredly in concert with AI assistance. 
    
    We demonstrate the utility of our methods using an AI system developed to assist radiologists in the grading of spinal stenosis in lumbar MRI.
    Degenerative spinal stenosis is the abnormal narrowing of the spinal canal that compresses the spinal cord or nerve roots, often resulting in debility from pain, limb weakness, and other neurological disorders. 
    It is a highly prevalent condition that affects working-age and elderly populations and constitutes a heavy societal burden not only in the form of costly medical care but from decreased workplace productivity, disability, and lowered quality of life.
    The formal interpretation of spinal stenosis imaging remains a challenging and time-consuming task even for experienced subspecialty radiologists due to the complexity of spinal anatomy, pathology, and the MR imaging modality. Using a highly accurate AI model to help assess the severity of spinal stenosis on MRI could lower interpretation time and improve the consistency of grading. ~\cite{stenosisvariability} Yet, given the practical challenges of medical imaging that can degrade model performance in any given exam, the adoption of such tools will be low if clinicians encounter poor quality predictions without a sense of when the model is more or less reliable. To bolster such trust, we apply conformal prediction to algorithmic disease severity classification sets in order to identify higher uncertainty predictions that might merit special attention by the radiologist. 
    Our main contributions are the following:
    \begin{enumerate}
        \item We develop new distribution-free uncertainty quantification methods for ordinal labels.
        \item To our knowledge, we are the first to apply distribution-free uncertainty quantification to the results of an AI model for automated stenosis grading of lumbar spinal MRI.
        \item We identify a correlation between high  prediction uncertainty in individual cases and the presence of potentially contributory imaging features as evaluated by a neuroradiologist such as tumors, orthopedic hardware artifacts, and motion artifacts. 
    \end{enumerate}
    
    \begin{figure}[t]
        % \vspace{-0.3cm}
        \centering
        \includegraphics[width=\textwidth]{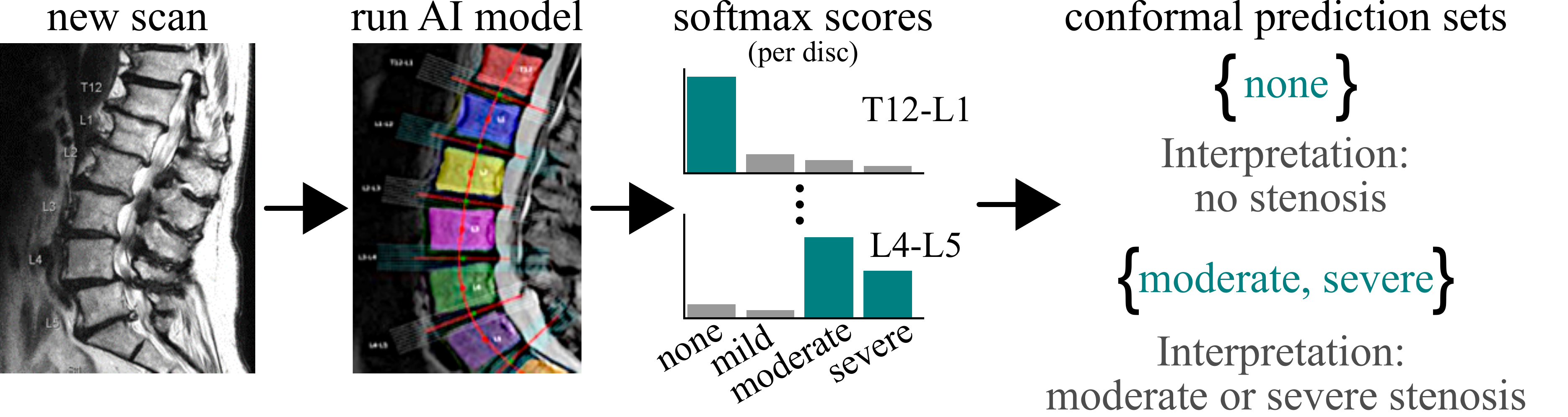}
        % \vspace{-0.3cm}
        \caption{The output of our conformal procedure on a lumbar spine MRI.} \label{fig:teaser}
    \end{figure}

\section{Methods}
    \label{sec:methods}
    To formally describe the problem, let the input $\Xtest \in \X$, $\X=\R^{H \times W \times D}$ be an MR image and the ground truth label $\Ytest \in \Y$, $\Y = \{0,...,K-1\}$ be an ordinal value representing the severity of the disease (higher values indicating greater severity).
    We are given a pre-trained model, $\hat{f}$, that takes in images and outputs a probability distribution over severities; for example, $\hat{f}$ may be a 3D convolutional neural network with a softmax function.
    Assume we also have a calibration dataset, $\big\{(X_i,Y_i)\big\}_{i=1}^{n}$, of data points that the model has not seen before.
    This calibration data should be composed of pairs of MR images and their associated disease severity labels drawn i.i.d.
    Given a new MR image $\Xtest$, the task is to predict the (unobserved) severity, $\Ytest$.
    In the usual multi-class classification setting, the output is the label with the highest estimated probability, $\Yhat(x) = \underset{y \in \{1,...,K\}}{\arg \max}\hat{f}(x)_y$.
    However, $\Yhat(\Xtest)$ may be wrong, either because the learned model $\hat{f}$ does not learn the relationship between MR images and severities properly or because there is intrinsic randomness in this relationship that cannot be accounted for by any algorithm (i.e. aleatoric uncertainty).
    
    Our goal is to rigorously quantify this uncertainty by outputting a set of probable disease severities that is guaranteed to contain the ground truth severity on average.
    These \emph{prediction sets} will provide \emph{distribution-free} probabilistic guarantees, \emph{i.e.} ones that do not depend on the model or distribution of the data.
    
    \subsubsection*{Ordinal Adaptive Prediction Sets (Ordinal APS) \newline}
        Our approach uses conformal prediction with a novel score function designed for ordinal labels.
        In particular, each prediction set will always be a contiguous set of severities, and for any user-specified error rate $\alpha$, prediction sets will contain the true label with probability $1-\alpha$.
        The reader can refer to~\cite{romano2020classification} and~\cite{sesia2021conformal} for similar algorithms and exposition.
            
        \subsubsection*{An Oracle Method \newline}
        Imagine we had oracle access to the true probability distribution over severities $P\big(\Ytest \mid \Xtest\big)$ with associated density function $f(x)_y$.
        A reasonable goal might then be to pick the set with the smallest size while still achieving \emph{conditional coverage}.
        
        \begin{definition}[conditional coverage]
            A predicted set of severities $\T(\Xtest)$ has conditional coverage if it contains the true severity with $1-\alpha$ probability no matter what MR image is input, i.e.,
            \begin{equation}
                \small
                \label{eq:conditional-coverage}
                \P\left( \Ytest \in \T(x) \mid \Xtest = x \right) \geq 1 - \alpha \text{, for all } x \in \X.
            \end{equation}
        \end{definition}
        
        The clinical benefit of conditional coverage is that it essentially achieves a per-patient guarantee as opposed to one that is averaged across patients.
        Under conditional coverage, the uncertainty sets will work equally well for all possible subgroups such as subpopulations from different patient demographics.
        
        Ignoring tie-breaking, we can succinctly describe the oracle prediction set as follows:
        \begin{equation}
            \small
            \label{eq:t-opt}
            \begin{aligned}
                \T^{(\rm optimal)}(x) &=  \left[l^*(x), \; u^*(x)\right], \text{ where } \\ \big(l^*(x), \; u^*(x)\big) &= \underset{\substack{(l,u) \in \Y^2 \\ l \leq u}}{\arg \min} \Bigg\{ u-l : \sum\limits_{j=l}^{u}f(x)_j \geq 1-\alpha \Bigg\}.
            \end{aligned}
        \end{equation}
        This set, $\T^{(\rm optimal)}$, is the smallest that satisfies~\eqref{eq:conditional-coverage}.
        Ideally, we would compute $\T^{(\rm optimal)}$ exactly, but we do not have access to $f$, only its estimator $\hat{f}$, which may be arbitrarily bad.

        \subsubsection*{Ordinal Adaptive Prediction Sets \newline}
        Naturally, the next step is to plug in our estimate of the probabilities, $\hat{f}$,  to~\eqref{eq:t-opt}.
        However, because $\hat{f}$ may be wrong, we must calibrate the resulting set with conformal prediction, yielding a \textit{marginal coverage} guarantee.
        Our procedure is illustrated graphically in the right plot of Figure~\ref{fig:ordinal-aps}; it corresponds to greedily growing the set outwards from the maximally likely predicted severity (i.e. ``mild'' stenosis in this example).
        \begin{definition}[marginal coverage]
            A predicted set of severities $\T$ has marginal coverage if it contains the true severity on average over new MRIs, i.e., 
            \begin{equation}
                \small
                \label{eq:marginal-coverage}
                \P\left( \Ytest \in \T(\Xtest) \right) \geq 1-\alpha.
            \end{equation}
        \end{definition}
        
        \begin{figure}[t]
            % \vspace{-0.3cm}
            \centering
            \includegraphics[width=0.7\textwidth]{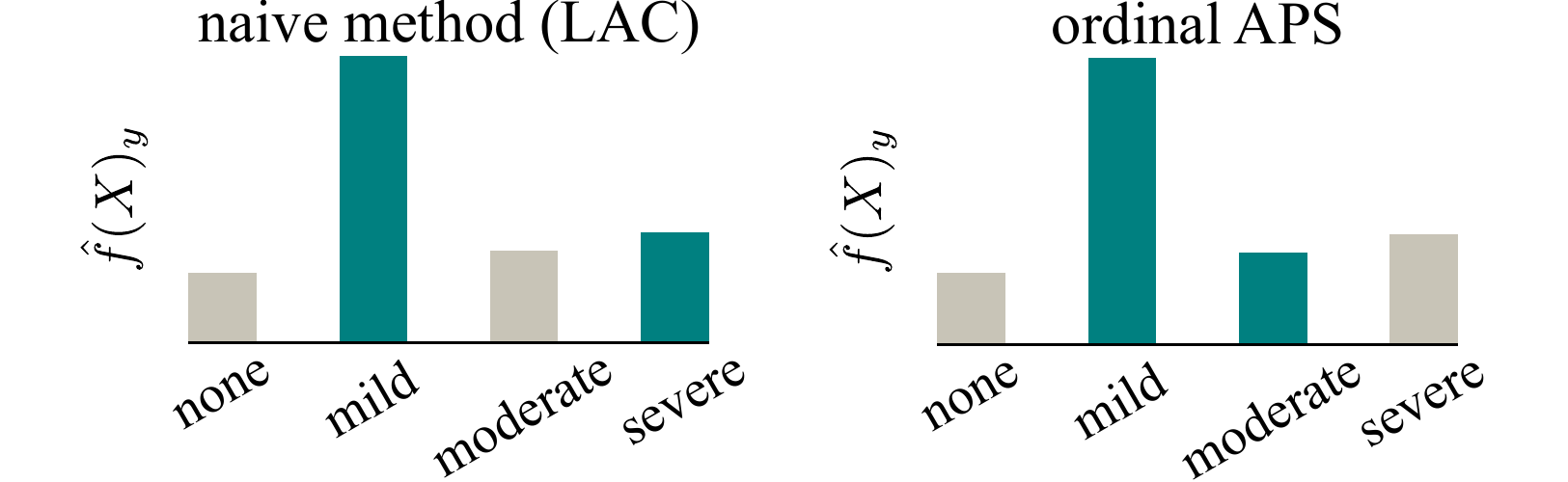}
            % \vspace{-0.3cm}
            \caption{Comparison of an example prediction set chosen by naive least-ambiguous set-valued (LAC) classifiers~\cite{Sadinle2016LeastAS} and Ordinal APS methods;
            notice that LAC does not respect ordinality, which may result in an incongruous prediction set.} \label{fig:ordinal-aps}
        \end{figure}
        
        Marginal coverage is weaker than conditional coverage since it holds only on average over the entire population, so coverage may be worse or better for some subgroups.
        While conditional coverage is, in general, impossible~\cite{barber2019limits}, we can hope to approximate conditional coverage by defining a set similar to $\T^{(\rm optimal)}$ that uses $\hat{f}$.
        To that end, define a sequence of sets indexed by a threshold $\lambda$,
        \begin{equation}
            \label{eq:t-opt-hat}
            \small
            \begin{aligned}
                \Tlam(x) &= \; [\hat{l}^*(x), \; \hat{u}^*(x)], \text{ where } \\ \big(\hat{l}^*(x), \; \hat{u}^*(x)\big) &= \underset{\substack{(l, \; u) \; \in \; \Y^2 \\ l \; \leq \; u}}{\arg \min} \Bigg\{ u-l : \sum\limits_{j=l}^{u}\hat{f}(j|x) \geq \lambda \Bigg\}.
            \end{aligned}
        \end{equation}
        
        Notice that as $\lambda$ grows, the sets grow, meaning they are \emph{nested} in $\lambda$:
        \begin{equation}
            \lambda_1 \leq \lambda_2 \implies \forall x, \; \; \T_{\lambda_1}(x) \subseteq \T_{\lambda_2}(x).
        \end{equation}
        
        The key is to pick a value of $\lambda$ such that the resulting set satisfies~\eqref{eq:marginal-coverage}.
        The following algorithm takes as input $\mathcal{T}_{\lambda}$ and outputs our choice of $\lambda$:
        \begin{equation}
            \label{eq:alg-calibrate-nested}
            \small
            \A(\mathcal{T}_{\lambda} ; \; \alpha) = \inf\Bigg\{ \lambda : \sum\limits_{i=1}^n \ind{Y_i \in \mathcal{T}_{\lambda}(X_i)}\; \geq \; \lceil(n+1)(1-\alpha)\rceil\Bigg\}.
        \end{equation}
        
        The key to our guarantee is the quantity $\lceil (n+1)(1-\alpha) \rceil$, which is slightly larger than the naive choice $n(1-\alpha)$ and helps us correct for the model's deficiencies; see~\cite{angelopoulos2021gentle} for details on this statistical argument.
        Using this algorithm, approximated in Algorithm~\ref{alg:approx-ordinal-aps}, results in a marginal coverage guarantee.
        \begin{theorem}[Conformal coverage guarantee]
            \label{thm:conformal-coverage}
            \small
            Let $(X_1,Y_1)$, $(X_2,Y_2)$, ..., $(X_n,Y_n)$ and $(\Xtest,\Ytest)$ be an i.i.d. sequence of MRIs and paired severities and let $\lhat=\A(\mathcal{T}_\lambda, \; \alpha)$. 
            Then $\mathcal{T}_{\lhat}$ satisfies~\eqref{eq:marginal-coverage}, i.e., it contains the true label with probability $1-\alpha$.
        \end{theorem}
        
        This theorem holds for any data distribution or machine learning model, any number of calibration data points $n$, and any possible sequence of nested sets that includes $\Y$ (see the formal version and proof in Appendix~\ref{app:proofs}).

        \subsubsection*{Implementing Ordinal Adaptive Prediction Sets}
        In practice, Ordinal APS has two undesirable properties: computing $\Tlam(x)$ requires a combinatorial search of the set of possible severities, and $\Tlam(x)$ may not include the point prediction $\hat{Y}$.
        In practice, we therefore approximate Ordinal APS greedily as described below in Algorithm~\ref{alg:approx-ordinal-aps}.
        
        \begin{algorithm}[ht]
\caption{Pseudocode for approximately computing $\Tlam(x)$}
\label{alg:approx-ordinal-aps}

\small

\textbf{Input:} Parameter $\lambda$; underlying predictor $\hat{f}$; input $x \in \X$.

\textbf{Output:} $\Tlam(x)$.

\begin{algorithmic}[1]

\State $\Tlam(x) \gets \arg \max\hat{f}(x)$
\State $q \gets 0$

\While{$q \leq \lambda$}
    \State $S \gets \{\min \Tlam(x) - 1, \; \max \Tlam(x) + 1\}$ \vspace{0.03cm}
    \State $y \gets \underset{y' \in S}{\arg \max} \, \hat{f}(x)\ind{y' \in \{1,...,K\} }$
    \State $q \gets q + \hat{f}(x)_y$
    \State $\Tlam(x) \gets \Tlam(x) \; \cup \; \{y\}$
\EndWhile
\end{algorithmic}

\end{algorithm}

        The algorithm always contains $\hat{Y}$ and requires only $\mathcal{O}(n)$ computations; furthermore, it usually results in exactly the same sets as the exact method in our experiments, which have a small value of $K$.
        
        Note that the approximate choice of $\Tlam(x)$ described in Algorithm~\ref{alg:approx-ordinal-aps} is still nested, and thus we can still guarantee coverage (see Corollary~\ref{cor:approx} for a formal statement and proof).
        
\section{Experiments}

    We compare Ordinal Adaptive Prediction Sets to two other conformal methods: \textit{Least Ambiguous set-valued Classifier} (LAC)~\cite{Sadinle2016LeastAS} and \textit{Ordinal Cumulative Distribution Function} (CDF).
    LAC uses the softmax score of the true class as the conformal score function.
    LAC theoretically gives the smallest average set size but sacrifices conditional coverage to achieve this. 
    Additionally, LAC does not respect ordinality and thus may output non-contiguous prediction sets, which are inappropriate in an ordinal setting such as disease severity rating.
    The Ordinal CDF method starts at the highest prediction score and then inflates the intervals by $\lambda$ in quantile-space; in that sense, it is similar to a discrete version of conformalized quantile regression~\cite{koenker1978regression,romano2019conformalized}.
    We only use the non-randomized versions of these methods.
    
    We evaluate these three conformal methods on a deep learning system previously developed for automated lumbar spine stenosis grading in MRI, DeepSPINE~\cite{DBLP:journals/corr/abs-1807-10215}. 
    The deep learning system consists of two convolutional neural networks -- one to segment out and label each vertebral body and disc-interspace and the other to perform multi-class ordinal stenosis classification for three different anatomical sites (the central canal and right and left neuroforamina) at each intervertebral disc level for a total of up to 18 gradings per patient.
    
    \begin{figure}[t]
         \centering
         \includegraphics[width=\textwidth]{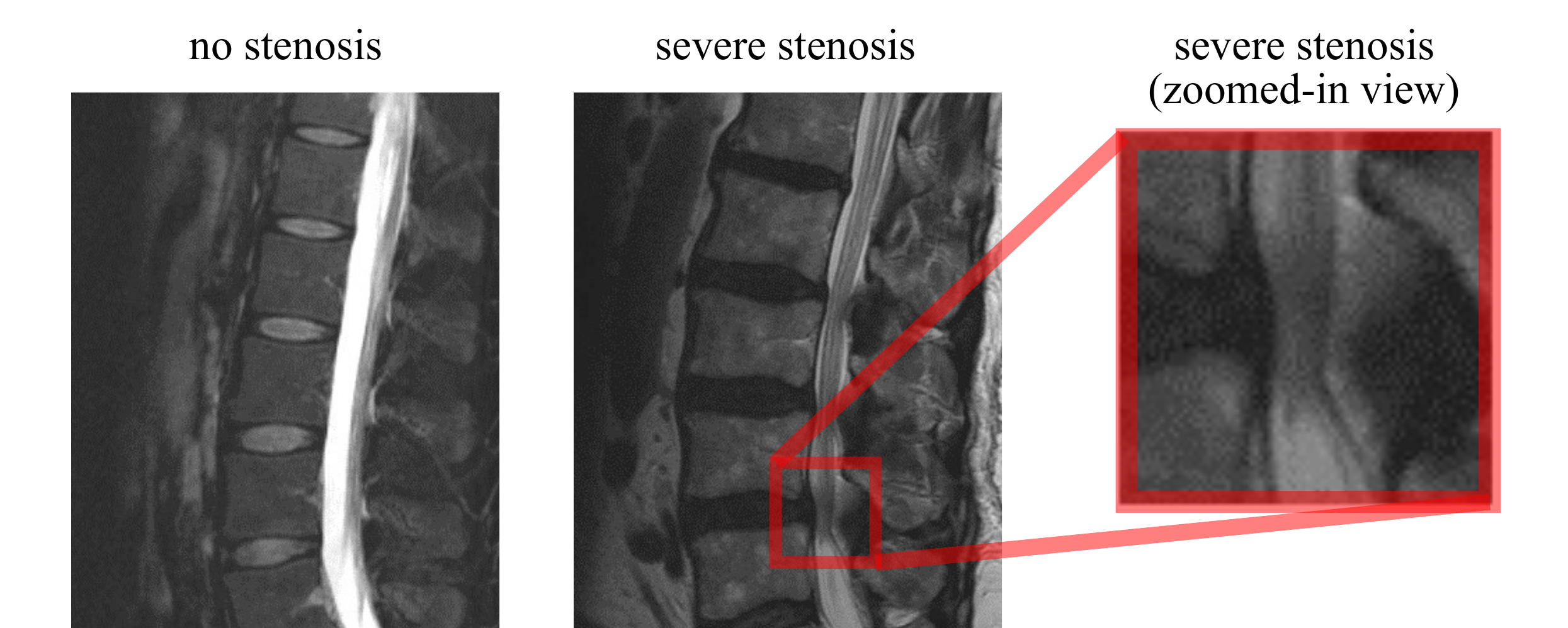}
         \caption{Example of a case without lumbar spinal stenosis and with severe lumbar spinal stenosis.} \label{fig:stenosis}
    \end{figure}
    
    For each MRI exam, the associated radiology report was automatically parsed for keywords indicative of the presence and severity of stenosis for each of the 6 vertebral disc levels (T12-L1 through L5-S1) to extract ground truth labels for a total of 6,093 gradings.
    Each grading was assigned to a value on a four-point ordinal scale of stenosis severity: 0 (\textit{no stenosis}), 1 (\textit{mild stenosis}), 2 (\textit{moderate stenosis}), and 3 (\textit{severe stenosis}). 
    Examples of patients with and without stenosis are shown in Figure~\ref{fig:stenosis}.
    
    For our experiments, we treat the stenosis grading model as a static model and only use it for inference to make predictions.
    We then process these predicted scores with the split-conformal techniques described in Section~\ref{sec:methods}.
    This scenario would most closely reflect the clinical reality of incorporating regulated, third-party AI software medical devices, which would likely not permit users access to the AI model beyond the ability to make predictions. 
    
    Our code and analysis used in the quantitative experiments are made available here: \texttt{https://github.com/clu5/lumbar-conformal}.
        
    \subsection{Quantitative Experiments}
        We use the predicted softmax scores from a held-out set of MRI exams from 409 patients, comprising 6,093 disc level stenosis severity predictions to calibrate and evaluate each conformal methods.
        We randomly include 5\% of patients in the calibration set and reserve the remainder for evaluating coverage and set size.
        We evaluate performance at several different $\alpha$ thresholds, $\alpha \in \{0.2, 0.15, 0.1, 0.05, 0.01\}$, and average results over 100 random trials.    
        
        \begin{figure}
            \centering
            \includegraphics[width=\textwidth]{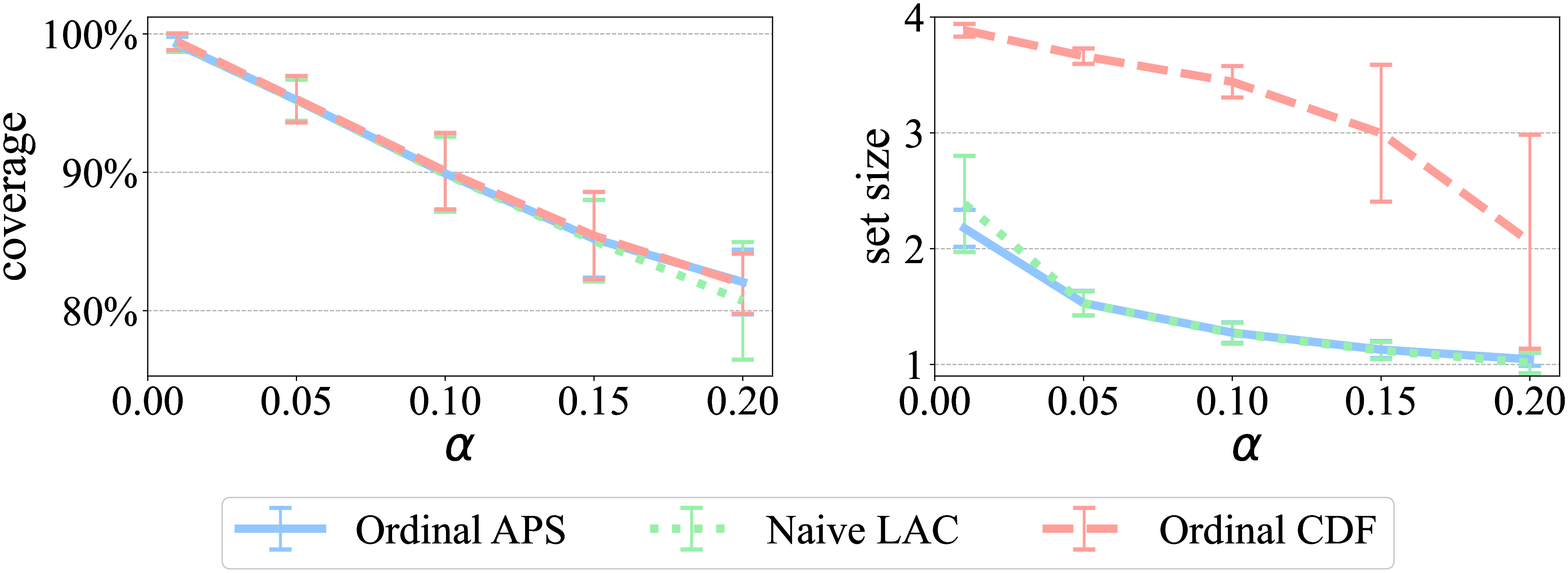}
            % \vspace{-0.3cm}
            \caption{Empirical coverage and set size for three conformal prediction methods for $\alpha \in \{0.2, 0.15, 0.1, 0.05, 0.01\}$ (averaged over 100 trials and shown with $\pm$ 1 standard deviation).} \label{fig:aggregate-compare}
        \end{figure}
        
        As expected, all three conformal methods empirically achieve the desired marginal coverage as guaranteed by Theorem~\ref{thm:conformal-coverage}.
        However, Ordinal CDF requires a much larger set size to attain proper coverage than either Naive LAC or Ordinal APS for all values of $\alpha$ (shown in Figure~\ref{fig:aggregate-compare}).
        
        In addition, while aggregate coverage is satisfied for each method, we find significant differences in class-conditional coverage (i.e. prediction sets stratified by the true stenosis severity label), which is shown in Figure~\ref{fig:class-conditional-compare}.
        We see that prediction sets for ``mild stenosis'' and ``moderate stenosis'' grades have lower average coverage and larger set sizes than prediction sets for ``no stenosis'' and ``severe stenosis'' grades.
        These differences may be partly attributed to the fact that the ``no stenosis'' class constitutes the majority of the label distribution (67\%) and so may be easier for a model trained on this distribution to classify.
        Additionally, ``mild stenosis'' and ``moderate stenosis'' grades may be more challenging to differentiate than ``no stenosis'' and ``severe stenosis'' grades, reflecting the greater inherent variability and uncertainty in the ground truth ratings.

        \begin{figure}
            \centering
            \includegraphics[width=\textwidth]{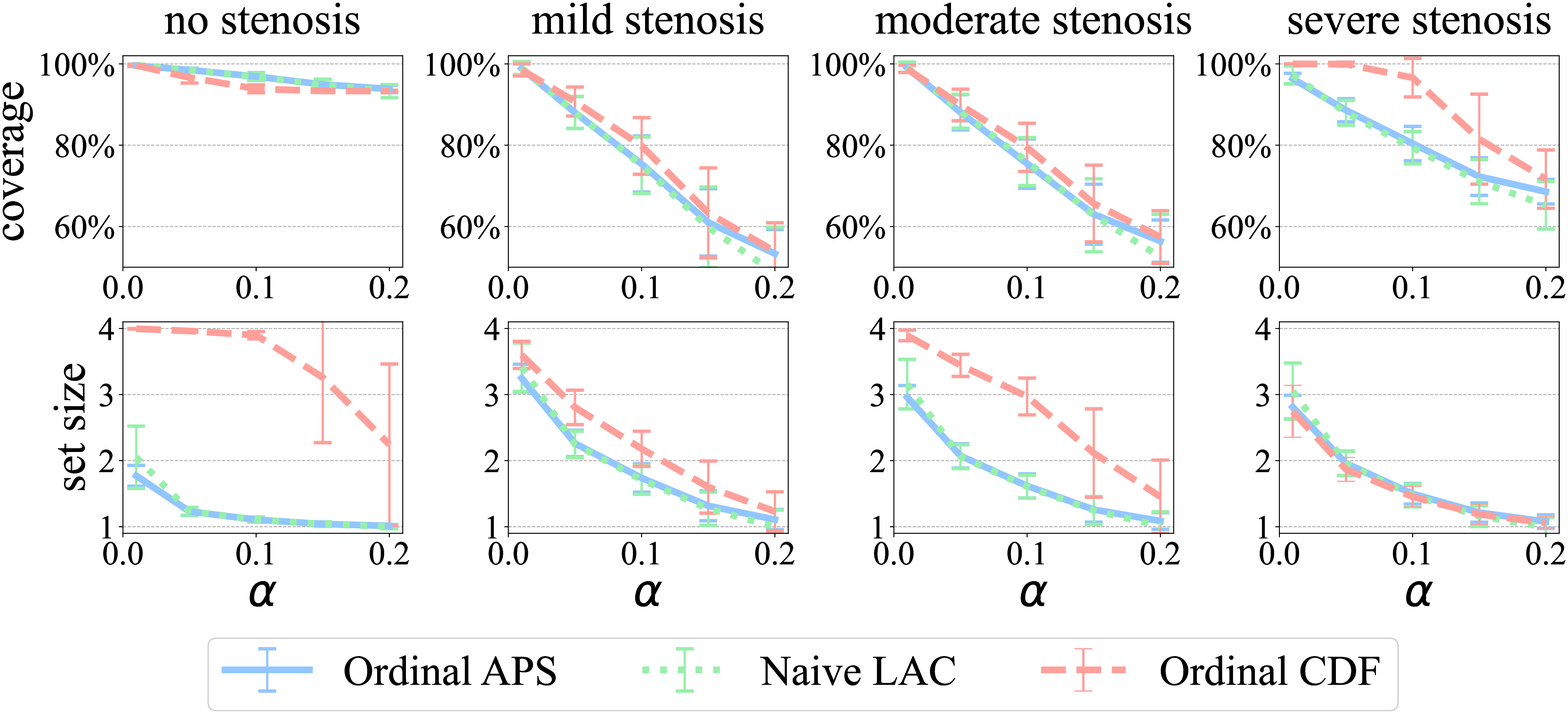}
            % \vspace{-0.3cm}
            \caption{Comparing coverage and set size when stratified by ground-truth stenosis severity grading.} \label{fig:class-conditional-compare}
        \end{figure}
        
        \begin{figure}
            \centering
            \includegraphics[width=\textwidth]{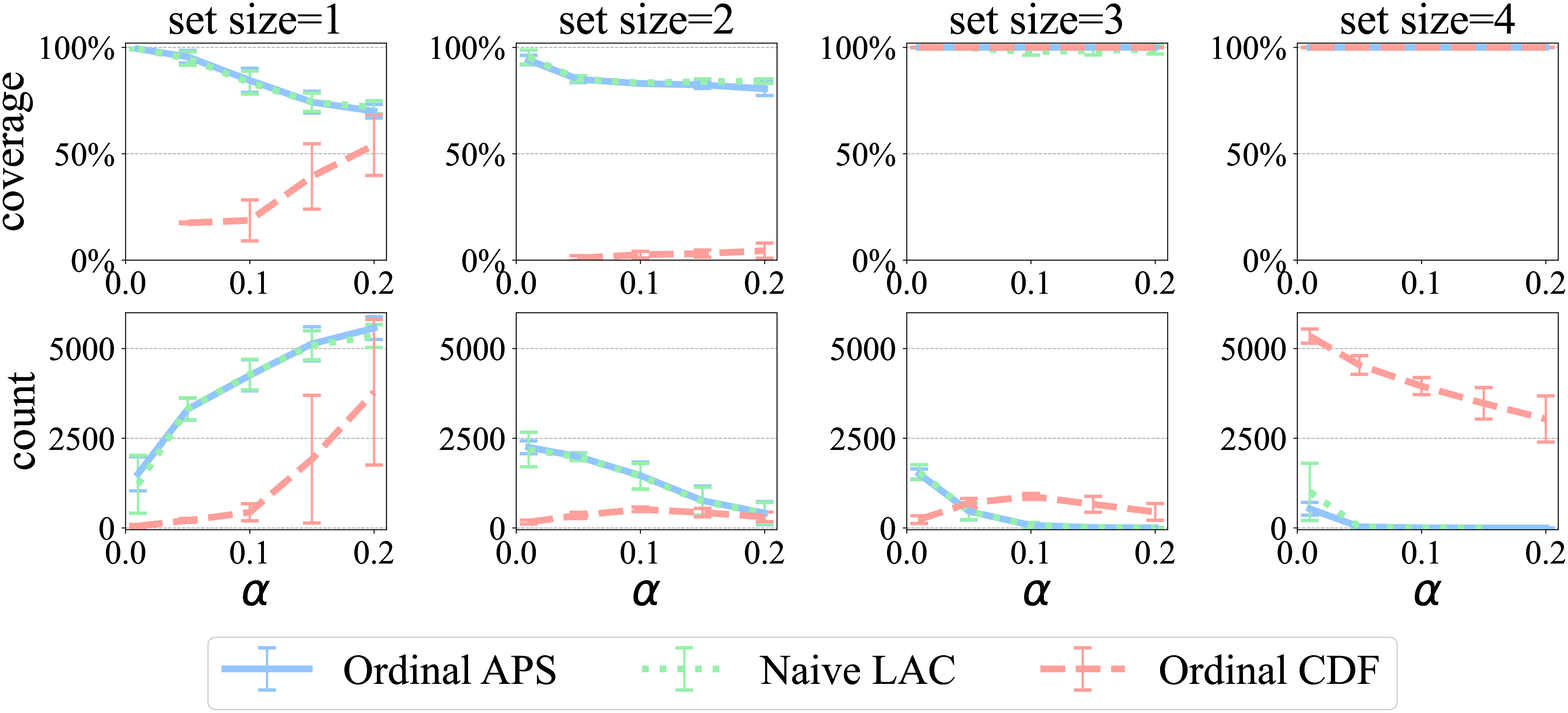}
            % \vspace{-0.3cm}
            \caption{Comparing coverage and count of number of sets with a particular size when stratified by prediction set size (coverage is only shown if there was at least one prediction set at the desired $\alpha$ for a particular method).} \label{fig:size-conditional-compare}
        \end{figure}
        
        We also compare coverage and distribution stratified by set size in Figure~\ref{fig:size-conditional-compare}.
        Stratifying by set size reveals that most predictions with Ordinal APS and Naive LAC contain only one or two grading predictions while Ordinal CDF mainly predicts a set of all four possible gradings (which always trivially satisfies coverage).

        Lastly, we compare coverage and set size at the disc level in Table~\ref{tab:disc-level-performance} at $\alpha = 0.1$.
        We find that coverage was inversely correlated to the prevalence of severe stenosis, which is most often found in the lower lumbar disc levels.
        
         \begin{table}[ht]
         \centering
         \caption{Coverage and set size stratified by intervertebral disc level at $\alpha=0.1$.}
         \label{tab:disc-level-performance}
         \small
         \begin{tabular}{c|cccc}
         \toprule
         \bf{disc level} & \bf{average severity} & \bf{method} & \bf{coverage} & \bf{set size} \\
         \midrule
         \multirow{3}{*}{\bf{T12-L1}} & \multirow{3}{*}{0.04}
             & Ordinal CDF & $99.7\% \pm 0.3\%$ & $ 3.97 \pm 0.02$ \\
             & & Naive LAC   & $98.6\% \pm 0.4\%$ & $ 1.04 \pm 0.02$ \\
             & & Ordinal APS & $98.6\% \pm 0.3\%$ & $ 1.04 \pm 0.01$ \\
         \midrule
         \multirow{3}{*}{\bf{L1-L2}} & \multirow{3}{*}{0.18}
             & Ordinal CDF & $97.5\% \pm 1.1\%$ & $ 3.84 \pm 0.06$ \\
             & & Naive LAC   & $95.4\% \pm 1.0\%$ & $ 1.12 \pm 0.04$ \\
             & & Ordinal APS & $95.4\% \pm 1.0\%$ & $ 1.12 \pm 0.04$ \\
         \midrule
         \multirow{3}{*}{\bf{L2-L3}} & \multirow{3}{*}{0.48}
             & Ordinal CDF & $90.4\% \pm 2.1\%$ & $ 3.50 \pm 0.11$ \\
             & & Naive LAC   & $90.0\% \pm 3.0\%$ & $ 1.25 \pm 0.09$ \\
             & & Ordinal APS & $90.1\% \pm 3.0\%$ & $ 1.26 \pm 0.09$ \\
         \midrule
         \multirow{3}{*}{\bf{L3-L4}} & \multirow{3}{*}{0.75}
             & Ordinal CDF & $84.3\% \pm 4.0\%$ & $ 3.17 \pm 0.19$ \\
             & & Naive LAC   & $86.0\% \pm 4.0\%$ & $ 1.42 \pm 0.12$ \\
             & & Ordinal APS & $85.7\% \pm 4.2\%$ & $ 1.42 \pm 0.13$ \\
         \midrule
         \multirow{3}{*}{\bf{L4-L5}} & \multirow{3}{*}{1.06}
             & Ordinal CDF & $81.7\% \pm 4.2\%$ & $ 2.86 \pm 0.21$ \\
             & & Naive LAC   & $83.2\% \pm 4.6\%$ & $ 1.48 \pm 0.15$ \\
             & & Ordinal APS & $83.2\% \pm 4.8\%$ & $ 1.48 \pm 0.15$ \\
         \midrule
         \multirow{3}{*}{\bf{L5-S1}} & \multirow{3}{*}{0.71}
             & Ordinal CDF & $84.4\% \pm 4.1$ & $ 3.20 \pm 0.17$ \\
             & & Naive LAC   & $85.3\% \pm 3.2$ & $ 1.39 \pm 0.12$ \\
             & & Ordinal APS & $85.7\% \pm 3.0$ & $ 1.41 \pm 0.11$ \\
         \midrule
         \midrule
         \multirow{3}{*}{\bf{Total}} & \multirow{3}{*}{0.54}
             & Ordinal CDF & $90.0\% \pm 2.5\%$ & $ 3.44 \pm 0.12$ \\
             & & Naive LAC   & $90.0\% \pm 2.6\%$ & $ 1.28 \pm 0.09$ \\
             & & Ordinal APS & $90.1\% \pm 2.7\%$ & $ 1.28 \pm 0.09$ \\
         \bottomrule
         \end{tabular}
        \end{table}
        
        Overall, we conclude that Ordinal APS performs similarly to LAC in both coverage and set size, and both Ordinal APS and Naive LAC outperform Ordinal CDF.
        The similarities between LAC and Ordinal APS are notable --- they almost always result in the same sets, although the algorithms are quite different.
        This is unsurprising given that in our setting $|\Y|=4$ and the model's accuracy is high, so bimodal softmax scores almost never happen.
        Therefore LAC and Ordinal APS do the same thing.
        This observation does not generalize; other ordinal prediction problems with more categories will have bimodal distributions and thus LAC and Ordinal APS will differ.
        % More results can be found in Appendix~\ref{apx:experiments}.
        % Code used in our experiments will be provided in a code repository. 
    
    \subsection{Clinical Review of High Uncertainty Predictions}

        To investigate the clinical utility of Ordinal APS to enhance AI-augmented workflows, we evaluate one possible clinical integration use case: flagging low confidence predictions (i.e. ones with a large set size). The radiologist's performance and user experience of the model may be improved by raising their awareness of those instances in which the model performance may be degraded by scan anomalies or when uncertainty quantification is very high, excluding such potentially poor quality results from their review responsibilities altogether. 
        We define an uncertainty score for each patient by taking the average set size for all disc levels and grading tasks.
        An neuroradiologist with $>20$ years of experience determined what constituted a ``significant imaging anomaly'' within the context of spinal stenosis interpretation.
            
         \begin{figure}[t]
            % \vspace{-0.3cm}
            \centering
            \includegraphics[width=\textwidth]{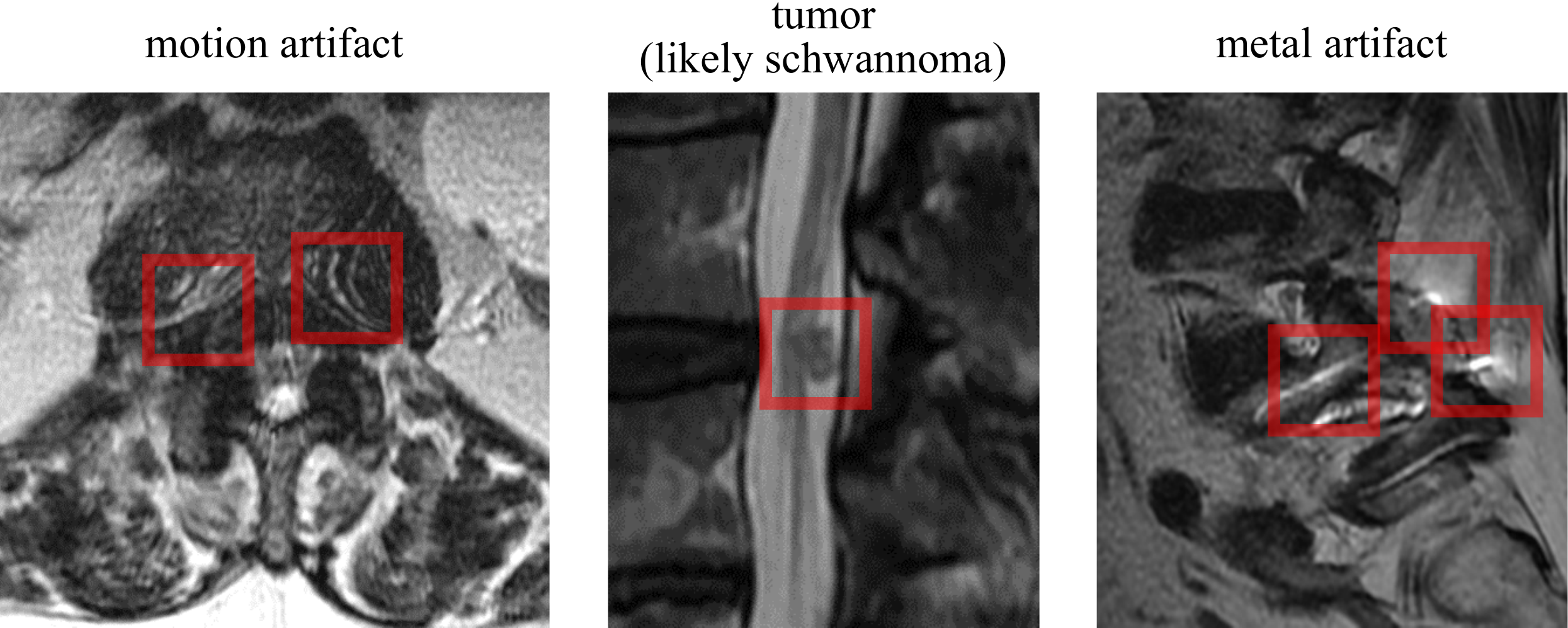}
            % \vspace{-0.8cm}
            \caption{Three anomalies found in high uncertainty predictions; the anomalous areas are boxed in red.} \label{fig:anomaly}
            % \vspace{.3cm}
        \end{figure} 
        
        As a statistical validation of these results, we examined the report of 70 cases with the highest uncertainty and found 17 such anomalies: 11 cases with artifacts from metallic orthopedic hardware, four cases with motion artifacts, one case with a large tumor occupying the spinal canal, and one case with a severe congenital abnormality (achondroplastic dwarfism). In contrast, a random sample of 70 cases from the dataset only demonstrated five cases with significant anomalies which were all orthopedic hardware artifacts.
        This difference is significant with $p<0.05$ by Fisher's exact test, and qualitatively the abnormalities found in the filtered samples were more extreme.

        Our manual review of high uncertainty cases shows promise for improving the clinician experience with AI-assisted medical software tools using distribution-free uncertainty quantification.  Rather than presenting all AI predictions as equally trustworthy, cases with higher uncertainty can be flagged to prompt additional scrutiny or hidden from the user altogether to maximize efficiency. While prospective evaluation of this use of conformal prediction in more clinically realistic settings will be needed to validate general feasibility and utility, our preliminary experiments are a step towards demonstrating the clinical applicability of conformal prediction.
        
\section{Related Work}
    \subsubsection*{Conformal Prediction and Distribution-Free Uncertainty Quantification}
    Conformal prediction is a flexible technique for generating prediction intervals from arbitrary models.
    It was first developed by Vladimir Vovk and collaborators in the late 1990s~\cite{vovk1999machine,vovk2005algorithmic,lei2013conformal,lei2013distribution,lei2014classification,Sadinle2016LeastAS}.
    We build most directly on the work of Yaniv Romano and collaborators, who developed the Adaptive Prediction Sets method studied in~\cite{romano2020classification,angelopoulos2020sets} and the histogram binning method in~\cite{sesia2021conformal}.
    The latter work is the most relevant, and it proposes an algorithm very similar to Algorithm~\ref{alg:approx-ordinal-aps} in a continuous setting with histogram regression.
    Our work also relies on the nested set outlook on conformal prediction~\cite{gupta2021nested}.
    We also build directly on existing work involving distribution-free risk-controlling prediction sets and Learn then Test~\cite{bates2021distribution,angelopoulos2021learn}.
    The LAC baseline is taken from~\cite{Sadinle2016LeastAS}, and the ordinal CDF baseline is similar to the softmax method in~\cite{angelopoulos2022image}, which is in turn motivated by~\cite{koenker1978regression,romano2019conformalized}.
    % Finally, we remark that our false alarm control example has some similarities to topics like outlier detection as discussed in~\cite{laxhammar2015inductive,smith2015conformal,bates2021multiple}.
    A gentle introduction to these topics and their history is available in~\cite{angelopoulos2021gentle}, or alternatively, in~\cite{shafer2008tutorial}.
    
    \subsubsection*{Uncertainty Quantification for Critical Applications} 
    Recently, uncertainty quantification has been promoted to facilitate trustworthiness and transparency in black-box algorithms, such as deep learning, for critical decision-making~\cite{10.1145/3461702.3462571}.
    In particular, conformal prediction methods have been applied to a wide range of safety-critical applications -- from reducing false alarms in the detection of sepsis risk \cite{sepsis} to end-to-end autonomous driving control \cite{DBLP:journals/corr/abs-1909-09884}.
    Distribution-free uncertainty quantification techniques such as conformal prediction sets have emerged as an essential tool for  rigorous statistical guarantees in medical decision-making~\cite{uncertaintyQuantificationMedical,lu2022fair,angelopoulos2021private,10.1007/s41666-021-00113-8,fannjiang2022conformal,https://doi.org/10.48550/arxiv.2206.12008}.
    
\section{Conclusion}
    We show how conformal prediction sets can complement existing AI systems without further modification to the model to provide distribution-free reliability guarantees.
    We demonstrate its clinically utility in the application of flagging high uncertainty cases in automated stenosis severity grading for followup review.
    We hope this work promotes further studies on the trustworthiness and usability of uncertainty-aware machine learning systems for clinical applications.
    
\section*{Acknowledgements}
SP receives research support from GE Healthcare. 
AA is funded by the NSFGRFP and a Berkeley Fellowship.
We thank Stephen Bates for many helpful conversations.
    
\bibliographystyle{splncs04} % We choose the "plain" reference style
\bibliography{bibliography} % Entries are in the refs.bib file
\clearpage
\appendix
\section{Formal Theorem Statements and Proofs}
\label{app:proofs}
  
  \begin{manualtheorem}{1}[Formal conformal coverage guarantee]
    \label{thm:conformal-coverage-formal}
    \small
    Let $(X_1,Y_1)$, $(X_2,Y_2)$, ..., $(X_n,Y_n)$ and $(\Xtest,\Ytest)$ be drawn independently and identically distributed from distribution $\P$, and let $\mathcal{C}_\lambda$ be any sequence of sets nested in $\lambda$ such that $\underset{\lambda \to \infty}{\lim} \mathcal{C}_\lambda = \Y$. 
    Finally, let $\lhat=\A(\mathcal{C}_\lambda, \; \alpha)$. 
    Then $\mathcal{C}_{\lhat}$ satisfies~\eqref{eq:marginal-coverage}.
  \end{manualtheorem}
  \begin{proof}[Theorem 1]
    This is a standard result in the conformal literature~\cite{vovk1999machine}.
    A stripped down proof appears in~\cite{angelopoulos2021gentle}.
    The nested set version appears in~\cite{gupta2021nested}.
  \end{proof}
  
  The informal version of Theorem 1 in the main text is simply a corollary of the above when applied to the specific sequence of nested sets $\Tlam$.
  
  \begin{proof}[Corollary 1]
    Pick $\lambda^{(1)} \leq \lambda^{(2)}$.
    It is clear that $\hat{f}(j|x) \geq \lambda^{(2)}$ implies $\hat{f}(j|x) \geq \lambda^{(1)}$.
    Therefore $\T_{\lambda^{(2)}}(x) \subseteq \T_{\lambda^{(1)}}(x)$.
    Applying Theorem~\ref{thm:conformal-coverage} completes the proof.
  \end{proof}
  
  \begin{cor}
    \label{cor:approx}
    \small
    In the setting of Theorem 1, let $\tilde{\mathcal{T}}_{\lambda}$ be the sequence of nested sets computed using Algorithm~\ref{alg:approx-ordinal-aps}, and let $\lhat=\A\Big(\tilde{\mathcal{T}}_{\lambda},\;\alpha\Big)$.
    Then $\tilde{\mathcal{T}}_{\lhat}$ satisfies~\ref{eq:marginal-coverage}.
  \end{cor}
  \begin{proof}[Corollary~\ref{cor:approx}]
    Pick $\lambda^{(1)} \leq \lambda^{(2)}$.
    Examining Algorithm~\ref{alg:approx-ordinal-aps}, $q\leq \lambda^{(1)}$ implies $q\leq \lambda^{(2)}$ .
    Therefore $\tilde{\T}_{\lambda^{(1)}}(x) \subseteq \tilde{\T}_{\lambda^{(2)}}(x)$.
    Applying Theorem~\ref{thm:conformal-coverage} completes the proof.
  \end{proof}

\end{document}